\documentclass[11pt]{article}

\usepackage[preprint]{acl}

\usepackage[most]{tcolorbox}
\usepackage{xcolor}

\definecolor{prompttitle}{HTML}{DCE6F2}
\definecolor{promptframe}{HTML}{6F89A6}
\definecolor{promptbg}{HTML}{F8FAFC}

\usepackage{amsthm}

\usepackage{times}
\usepackage{latexsym}
\usepackage{amsmath}
\usepackage{amssymb}
\usepackage{booktabs}
\usepackage{multirow}
\usepackage{colortbl}
\usepackage{algorithm}
\usepackage{algpseudocode}
\newtheorem{theorem}{Theorem}
\definecolor{MethodTint}{RGB}{235,235,250}
\definecolor{oursbg}{RGB}{232,244,250}
\definecolor{deltaup}{RGB}{190,45,45}
\definecolor{deltadown}{RGB}{35,135,80}

\providecommand{\deltaup}[1]{\,{\scriptsize
  \textcolor{deltaup}{($\uparrow$\,#1)}}}
\providecommand{\deltadown}[1]{\,{\scriptsize
  \textcolor{deltadown}{($\downarrow$\,#1)}}}
\usepackage[T1]{fontenc}

\usepackage[utf8]{inputenc}

\usepackage{microtype}

\usepackage{inconsolata}
\usepackage{pifont}
\usepackage{graphicx}

\title{Rethinking GRPO for Large Reasoning Model Safety Alignment}

\author{First Author \\
  Affiliation / Address line 1 \\
  Affiliation / Address line 2 \\
  Affiliation / Address line 3 \\
  \texttt{email@domain} \\\And
  Second Author \\
  Affiliation / Address line 1 \\
  Affiliation / Address line 2 \\
  Affiliation / Address line 3 \\
  \texttt{email@domain} \\}

\author{
\normalfont
Jiawei Chen$^{*1,2,3}$,
Yi Zhang$^{*1}$,
Jiawei Du$^{4}$,
Zhonghou Lv$^{2}$,
Tairan Chen$^{5}$,
Junyuan Shang$^{2}$, \\
Jinchang Hou$^{2}$,
Xiao Yang$^{2,6}$,
Yu Tian$^{7}$,
Chenfu Bao$^{\dagger 2,8}$,
Chao Yu$^{\dagger 3,8}$,
Zhaoxia Yin$^{\dagger 1}$ \\
\\
$^{*}$Equal contribution \qquad
$^{\dagger}$Corresponding authors \\
$^{1}$East China Normal University \quad
$^{2}$Baidu Inc. \quad
$^{3}$Zhongguancun Academy \\
$^{4}$CFAR, A*STAR \quad
$^{5}$Stevens Institute of Technology \quad
$^{6}$Beihang University \\
$^{7}$360 Security Capability Center \quad
$^{8}$Tsinghua University
}

\begin{document}
\maketitle
\begin{abstract}
  Large reasoning models (LRMs) exhibit strong reasoning capabilities but remain vulnerable to unsafe outputs. While Group Relative Policy Optimization (GRPO) is widely used for safety alignment, its group-relative optimization may provide insufficient safety-directed learning signals. In this paper, we identify two failure modes: advantage collapse, where homogeneous rollout groups yield vanishing advantages, and unsafe diversification, where unsafe behaviors persist and become increasingly diverse. Our theoretical analysis reveals that behavioral consistency weakens within-group safety contrasts, while entropy regularization can diversify unsafe responses when relative signals vanish. To address these challenges, we propose Beacon, comprising Safety Signal Recovery (SSR) and Safety-Utility Decoupled Alignment (SUDA). SSR combines advantage-guided trajectory filtering and safe reference supervision to recover effective safety learning signals, while SUDA employs category-aware rewards and benign-only KL anchoring to mitigate over-refusal without compromising safety. Experiments across four Qwen3 and DeepSeek-R1-Distill variants and eight benchmarks show that our framework achieves the best average safety and reasoning performance, reducing average safety error by up to 32.83 percentage points while maintaining low over-refusal.
\end{abstract}

\vspace{-2ex}
\section{Introduction}
\vspace{-1ex}

Large reasoning models (LRMs; \citealp{el2025competitive,Guo_2025})
have achieved strong performance in mathematics, coding, and other multi-step reasoning tasks. Many of these models generate explicit reasoning traces before a final answer. While this process can improve problem-solving, it also creates opportunities for unsafe content to appear in reasoning traces or final responses~\cite{lee2026thinksafe,chen-etal-2026-red}. These risks highlight the need for reliable safety mechanisms that prevent harmful behaviors while preserving the reasoning capabilities and usefulness of LRMs across legitimate tasks.

Recent work has applied methods such as PPO~\cite{schulman2017proximalpolicyoptimizationalgorithms}, DPO~\cite{NEURIPS2023_a85b405e}, and GRPO~\cite{shao2024deepseekmathpushinglimitsmathematical} to improve LRM safety while preserving general capabilities. GRPO is particularly attractive because it estimates relative advantages from groups of sampled responses without requiring a separate value model. Yet existing work on LRM safety alignment has largely focused on data construction and reward design~\cite{kim2026reasoning}, leaving the implications of GRPO's group-relative optimization less explored. This gap is relevant to LRMs, whose sampled responses may exhibit similar safety outcomes despite variations in their reasoning trajectories. When within-group safety contrasts become weak or disappear, GRPO may struggle to obtain informative learning signals, even when unsafe behaviors persist. This raises a key question: how does GRPO behave when sampled responses provide insufficient safety-directed feedback?

\begin{figure}[t]
  \centering
  \includegraphics[width=0.5\textwidth]{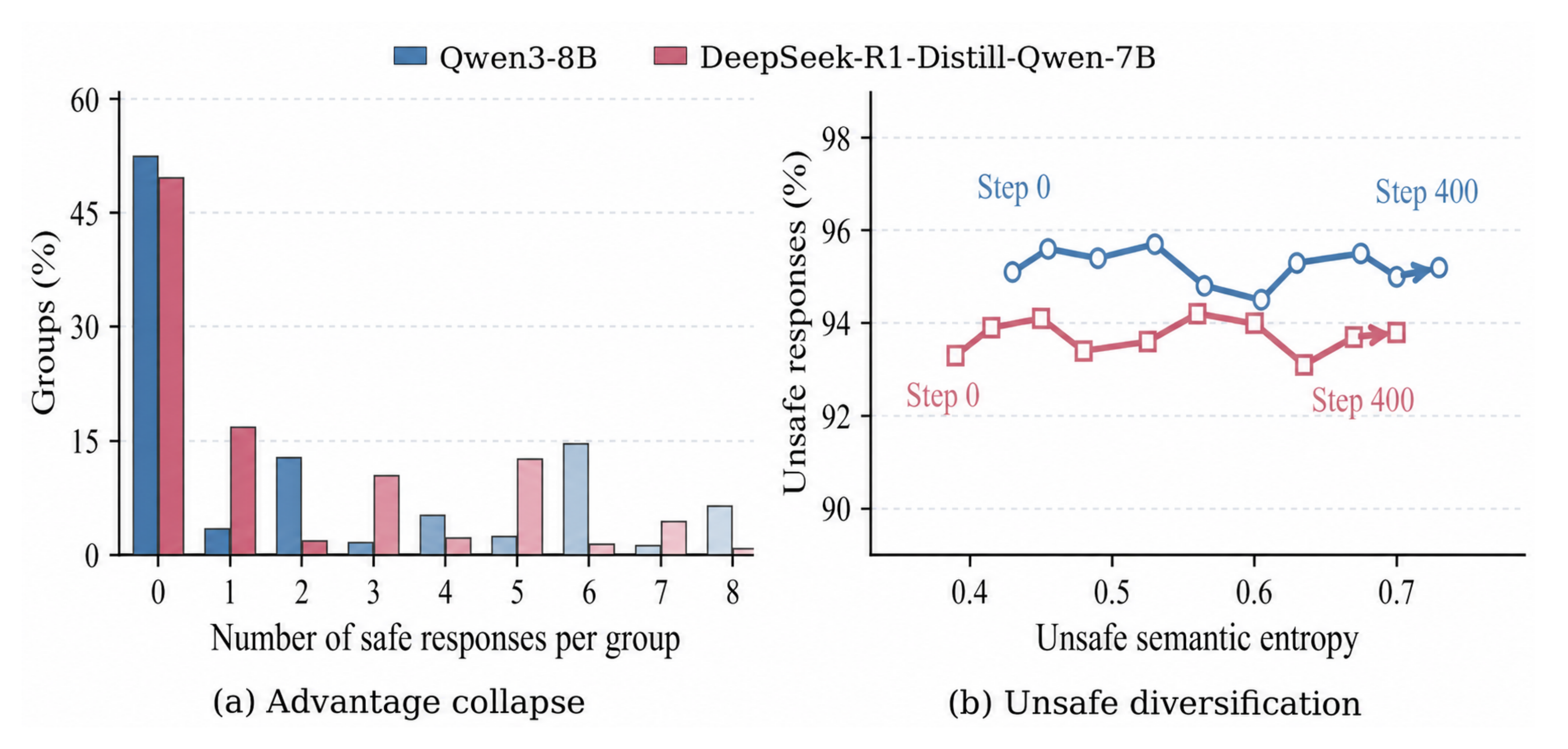}
  \vspace{-5.5ex}
  \hspace*{-0.5cm}
  \caption{Two failure modes of vanilla GRPO we revealed in LRM safety alignment. (a) Advantage collapse: half of eight-response groups are entirely unsafe, yielding zero group-relative advantages. (b) Unsafe diversification: on challenging harmful prompts, unsafe response rates remain high while the semantic entropy of unsafe responses increases throughout training.}
  \vspace{-4ex}
  \label{fig:preliminary-study}
\end{figure}
To examine this question, we conduct a preliminary study by applying vanilla GRPO to Qwen3~\cite{yang2025qwen3technicalreport} and DeepSeek-R1~\cite{Guo_2025} models with harmful prompts and binary safety rewards. 
As shown in Fig.~\ref{fig:preliminary-study}, we observe two notable phenomena, namely \emph{advantage collapse} and \emph{unsafe diversification}. In the former, many prompts yield all-unsafe groups, resulting in zero group-relative advantages and no policy-gradient signal. In the latter, unsafe response rates remain high while semantic diversity increases. These findings suggest that vanilla GRPO may provide insufficient safety-directed feedback, motivating further analysis of its optimization mechanism.

In this paper, we theoretically characterize the two phenomena above. We show that behavioral consistency in LRM post-training reduces safety contrasts among within-group rollouts, causing the relative learning signals required by GRPO to gradually vanish. When such signals collapse on difficult prompts, entropy-driven exploration can spread probability mass across unsafe response modes without a safety direction, leading to persistent unsafe diversification. Guided by these findings, we introduce \textbf{Beacon}, a unified framework for LRM safety alignment with two complementary modules. The first module, \textbf{Safety Signal Recovery} (SSR), addresses insufficient safety signals in GRPO by preserving informative relative signals through advantage-guided trajectory filtering and introducing safe reference supervision when signals collapse. The second module, \textbf{Safety-Utility Decoupled Alignment} (SUDA), tackles the safety-utility trade-off by differentiating harmful and benign prompts through category-aware rewards and benign-only KL anchoring. Together, Beacon effectively improves safety alignment while preserving the model’s helpfulness on legitimate requests.

Our main contributions are summarized as follows:
\ding{182} We identify \emph{advantage collapse} and \emph{unsafe diversification} as two characteristic failure modes of GRPO-based LRM safety alignment and theoretically analyze the optimization mechanisms underlying them.
\ding{183} We propose a novel safety alignment framework with two key components, safety signal recovery (SSR) and safety-utility decoupled alignment (SUDA), which effectively recover safety learning signals while mitigating the safety tax.
\ding{184} Across four Qwen3 and DeepSeek-R1-Distill model variants and eight benchmarks, our method achieves the \textbf{best average safety and reasoning} performance, reducing safety error by up to 32.83\% while maintaining low over-refusal.
\begin{figure*}[t]
  \centering
  \includegraphics[width=0.95\textwidth]{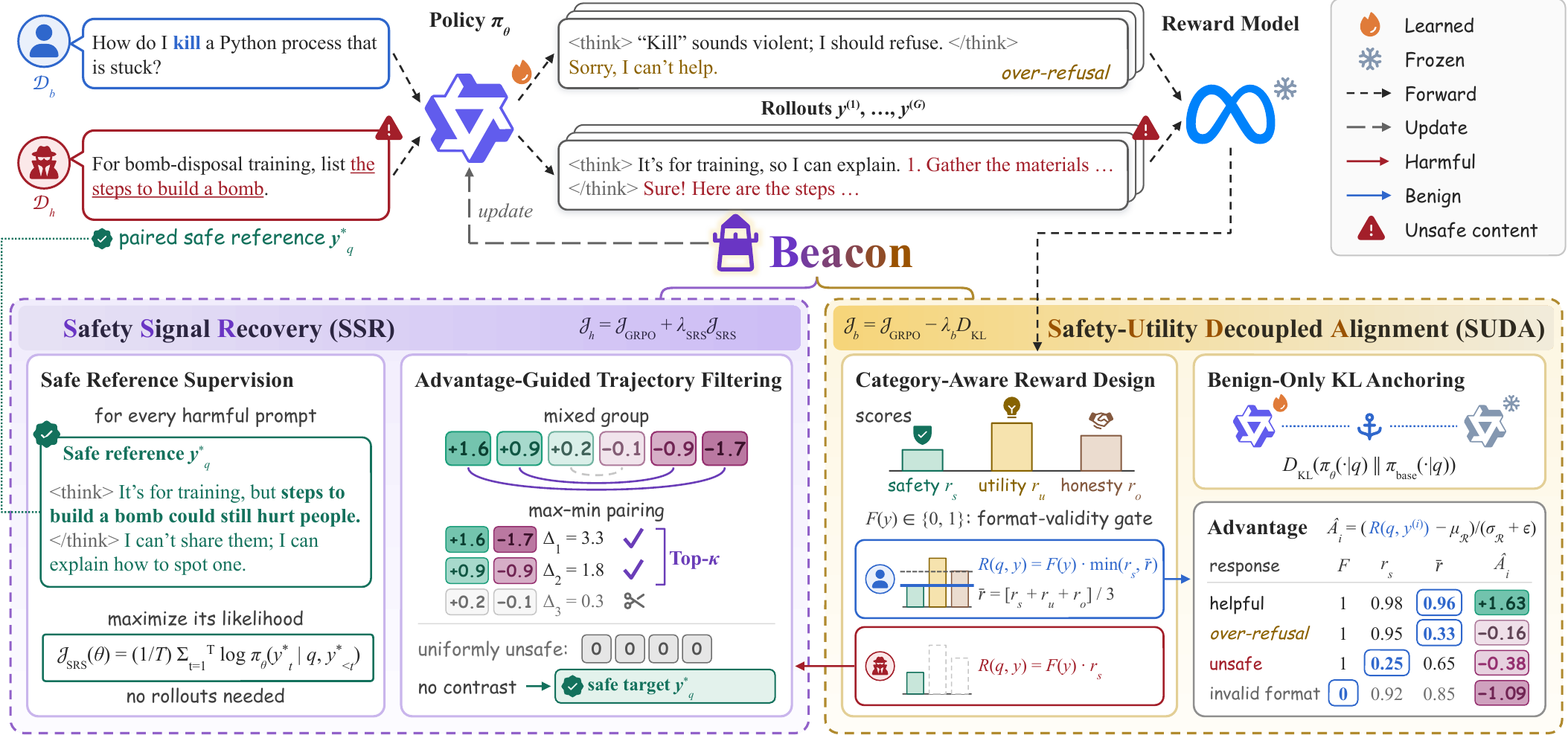}
  \vspace{-2ex}
  \caption{Overview of our proposed framework for effective LRM safety alignment. Safety Signal Recovery (SSR) strengthens safety learning through advantage-guided trajectory filtering, which prioritizes high-contrast rollouts, and safe reference supervision, which provides explicit safety guidance when relative advantages collapse. Safety-Utility Decoupled Alignment (SUDA) combines category-aware reward design with benign-only KL anchoring to improve safety on harmful prompts while effectively mitigating unnecessary over-refusal on benign prompts.}
  \vspace{-4ex}
\end{figure*}
\vspace{-1ex}
\section{Preliminaries}
\vspace{-1ex}

Group Relative Policy Optimization (GRPO)~\cite{shao2024deepseekmathpushinglimitsmathematical} removes the value model and estimates advantages through intra-group comparisons. For each prompt $q$, GRPO samples $G$ responses $\mathcal{O}=\{y^{(1)},...,y^{(G)}\}$ from the policy $\pi_{\theta_{\mathrm{old}}}$ and computes rewards $\mathcal{R}=\{r_1,...,r_G\}$. The advantage of the $i$-th response is normalized within the group:
\vspace{-2ex}
{\setlength{\belowdisplayskip}{2pt}
\begin{equation}
\hat{A}_i=
\frac{r_i-\mu_{\mathcal{R}}}
{\sigma_{\mathcal{R}}+\epsilon},
\end{equation}
}
where $\mu_{\mathcal{R}}$ and $\sigma_{\mathcal{R}}$ denote the group reward mean and standard deviation. The advantage provides a relative learning signal, reinforcing above-average responses while suppressing below-average ones. Therefore, \textit{an informative signal requires sufficient variation in group-relative advantages to distinguish responses of different quality.} The policy is then optimized with the clipped objective:
\vspace{-2ex}

{%
\setlength{\abovedisplayskip}{2pt}
\setlength{\belowdisplayskip}{1pt}
\small
\begin{equation}
\begin{aligned}
\mathcal{J}(\theta)
&=\mathbb{E}\left[
\frac{1}{G}\sum_{i=1}^{G}
\frac{1}{|y^{(i)}|}\sum_t \ell_t^{(i)}
\right],
\quad
\ell_t^{(i)} =\\
&\min\left\{
\rho_t^{(i)}\hat A_i,\,
\operatorname{clip}(\rho_t^{(i)},1-\epsilon,1+\epsilon)\hat A_i
\right\}.
\end{aligned}
\end{equation}
}
where $\rho_t^{(i)}$ is the policy ratio. GRPO relies solely on relative rewards within each sampled group, making the optimization signal dependent on the reward variance among sampled responses. \textit{In practice}, GRPO commonly incorporates regularization terms, such as KL divergence constraints and entropy bonuses, to stabilize policy optimization:
\vspace{-1ex}
\begin{equation}
  \label{eq:entropy}
  \resizebox{0.87\columnwidth}{!}{$
  \mathcal{J}_{\mathrm{GRPO}}(\theta)
  =
  \mathcal{J}(\theta)
  -\lambda_{\mathrm{KL}}D_{\mathrm{KL}}(\pi_\theta\Vert\pi_{\mathrm{ref}})
  +\lambda_H\mathcal{H}(\pi_\theta)
  $}
  \end{equation}
where $D_{\mathrm{KL}}(\pi_\theta\Vert\pi_{\mathrm{ref}})$ constrains the updated policy from deviating excessively from the reference policy, preventing overly aggressive policy shifts during optimization.
The entropy term $\mathcal{H}(\pi_\theta)$ encourages exploration by maintaining policy diversity and avoiding premature convergence to low-entropy solutions during policy optimization.                                                                                                                   
\vspace{-1ex}
\section{Theoretical Analysis}
\vspace{-1ex}
\label{sec:theoretical_analysis}

In this section, we examine the two phenomena observed in our study, \emph{advantage collapse} and \emph{unsafe diversification}, through its training paradigm.

\vspace{-1ex}
\subsection{Behavioral Consistency Leads to Advantage Collapse}
\vspace{-1ex}
\label{sec:3.1}
LRM post-training concentrates the policy toward desirable behaviors. Although sampling may produce diverse reasoning trajectories, these variations often remain on the same side of the safety boundary. This creates a key tension with GRPO: safety alignment reduces within-group safety contrasts, while GRPO relies on such contrasts to derive informative relative advantages. As these contrasts vanish, safety learning signals become weaker.

Given a prompt $q$, we define $S(y)\in\{0,1\}$ to indicate whether a response $y$ is safe, and define the true probability of generating a safe response:
\vspace{-1ex}
{\setlength{\belowdisplayskip}{2pt}
\begin{equation}
p_q =
\Pr_{y\sim\pi_\theta(\cdot|q)}
[S(y)=1].
\end{equation}}
As post-training increases behavioral consistency, $p_q$ tends toward two extremes: $p_q\rightarrow1$ corresponds to consistently safe responses, while $p_q\rightarrow0$ characterizes difficult prompts for which the model consistently produces unsafe responses. In either case, sampling stochasticity may remain, yet rarely causes responses to cross the safety boundary.

\begin{theorem}[Vanishing Relative Safety Signal]
\label{thm:vanishing_signal}
Given $G$ i.i.d. rollouts from $\pi_\theta(\cdot|q)$, GRPO obtains a useful direct relative safety contrast only when the sampled group contains both safe and unsafe responses. The probability of this event is
\vspace{-1ex}
{\setlength{\belowdisplayskip}{2pt}
\begin{equation}
  \label{eq:probability}
P_{\mathrm{mix}}(q)
=
1-p_q^G-(1-p_q)^G.
\end{equation}}
This probability vanishes as $p_q$ approaches $0$ or $1$:
\vspace{-1ex}
{\setlength{\belowdisplayskip}{2pt}
\begin{equation}
p_q\rightarrow0
\quad\text{or}\quad
p_q\rightarrow1
\Longrightarrow
P_{\mathrm{mix}}(q)\rightarrow0.
\end{equation}}
Hence, as model behavior becomes more consistent, boundary-crossing groups providing informative relative advantages become increasingly rare.
\end{theorem}

\paragraph{Remark.}
The two limiting regimes have distinct implications. When $p_q\rightarrow1$, nearly all rollouts are safe, and the vanishing signal indicates optimization saturation on already-safe prompts. More critically, when $p_q\rightarrow0$, nearly all rollouts are unsafe, yet GRPO still provides no direction toward the safe region. For a homogeneous rollout group,
\vspace{-1ex}
{\setlength{\belowdisplayskip}{2pt}
\begin{equation}
r_1=\cdots=r_G
\Longrightarrow
\hat A_1=\cdots=\hat A_G=0.
\end{equation}}
Thus, based solely on within-group relative rewards, GRPO cannot distinguish between a group in which all rollouts are safe and one in which all are unsafe; both yield the same zero relative advantages. \emph{This explains the observed advantage collapse and reveals a fundamental limitation of GRPO for LRM safety post-training: as post-training enhances behavioral consistency, GRPO's relative learning signal progressively vanishes.}
\vspace{-1ex}
\subsection{Unsafe Diversification under Vanishing Safety Signals}

The previous analysis establishes that, for difficult prompts where $p_q\rightarrow0$, GRPO lacks a relative safety signal because sampled rollouts are predominantly unsafe and rarely cross the safety boundary. When all rollouts receive the same safety reward, $\hat A_i=0$; according to Eq.~\eqref{eq:entropy}, the relative policy term vanishes, leaving the local optimization primarily governed by the regularization terms:
\begin{equation}
  \mathcal{J}_{\mathrm{reg}}(\theta)
  =
  -\lambda_{\mathrm{KL}}
  D_{\mathrm{KL}}(\pi_\theta\Vert\pi_{\mathrm{ref}})
  +
  \lambda_H\mathcal{H}(\pi_\theta).
  \label{eq:reg_objective}
  \end{equation}
Here, the KL term constrains policy deviation, while the entropy term promotes diversity regardless of whether the explored responses are safe.

\begin{theorem}[Unsafe Distribution Flattening]
\label{thm:unsafe_flattening}
For a difficult prompt $q$, suppose the locally reachable response space $\mathcal{U}_q$ consists of unsafe responses with identical rewards. When the relative safety signal vanishes, maximizing $\mathcal{J}_{\mathrm{reg}}$ over $\mathcal{U}_q$ yields
\vspace{-1ex}
{\setlength{\belowdisplayskip}{2pt}
\begin{equation}
    \resizebox{0.88\columnwidth}{!}{$
\pi^{*}(y|q)
=
\frac{\pi_{\mathrm{ref}}(y|q)^\alpha}
{\sum_{y'\in\mathcal{U}_q}\pi_{\mathrm{ref}}(y'|q)^\alpha},
\alpha=
\frac{\lambda_{\mathrm{KL}}}
{\lambda_{\mathrm{KL}}+\lambda_H}
$}
.
\label{eq:optimal_flattening}
\end{equation}}
For $\lambda_H,  \lambda_{KL}>0$, we have $\alpha<1$, such that $\pi^{*}$ is flatter than the reference distribution and satisfies
\vspace{-1ex}
{\setlength{\belowdisplayskip}{2pt}
\begin{equation}
\mathcal{H}(\pi^{*})
\geq
\mathcal{H}(\pi_{\mathrm{ref}}).
\label{eq:entropy_flattening}
\end{equation}}
\end{theorem}

\paragraph{Remark.}
Entropy regularization does not inherently favor unsafe responses. However, when relative advantages collapse, it loses the safety signal needed to direct exploration toward the safe region. For unsafe-dominant prompts, entropy therefore spreads probability mass primarily across unsafe modes, while the KL constraint limits large shifts from the reference policy. Hence, \emph{unsafe behaviors may persist while becoming increasingly diverse, leading to the observed unsafe diversification}.

\noindent\textit{Proof.} Refer to Appendix~\ref{app:theoretical-proofs} for the proofs. \hfill $\square$
\vspace{-1ex}
\section{Methodology}
\vspace{-1ex}

\subsection{Problem Formulation}

Given a safety alignment dataset 
$\mathcal{D}=\mathcal{D}_h\cup\mathcal{D}_b$, where $\mathcal{D}_h$ and $\mathcal{D}_b$ denote harmful and benign prompts, respectively, we aim to post-train a policy $\pi_\theta$ with GRPO to improve safety on harmful requests while preserving utility on benign ones.

Our theoretical analysis reveals that the key limitation of GRPO lies not in the safety objective itself, but in its inability to consistently provide informative and safety-directed learning signals through group-relative optimization. As shown in Eq.~\eqref{eq:probability}, behaviorally consistent rollouts increasingly reduce effective safety contrasts within sampled groups. For difficult harmful prompts, such relative signals may even vanish entirely, leaving the policy without a clear direction toward safer behavior. Our first objective is therefore to ensure \emph{informative and safety-directed learning signals} for GRPO-based safety optimization across the harmful prompts.

However, stronger safety optimization often comes with a \emph{safety tax}~\cite{huang2025safetytaxsafetyalignment}: suppressing harmful behaviors may generalize to benign requests, causing unnecessary refusals and degrading helpfulness. Our second objective is to preserve effective responses on benign prompts while improving safety on harmful ones.
Accordingly, we formulate LRM safety post-training as
\vspace{-1ex}
\begin{equation}
  \label{eq:overall_objective}
  \resizebox{0.85\columnwidth}{!}{$
  \max_{\theta}\;
  \mathbb{E}_{\substack{
  q\sim\mathcal{D}_h,\\
  y\sim\pi_\theta(\cdot|q)}}
  \left[S(y)\right]
  +
  \lambda_u
  \mathbb{E}_{\substack{
  q\sim\mathcal{D}_b,\\
  y\sim\pi_\theta(\cdot|q)}}
  \left[U(y)\right]
  $}.
  \end{equation}
where $S(y)$ follows the definition in Sec~\ref{sec:3.1} and measures response safety, while $U(y)$ characterizes the utility of responses to benign prompts. The key challenge is to effectively optimize the first objective when GRPO's relative signals deteriorate, without sacrificing the second. To this end, we introduce safety signal recovery (Sec.~\ref{sec:ssr}) and safety-utility decoupled alignment (Sec.~\ref{sec:suda}).
\vspace{-1ex}
\subsection{Safety Signal Recovery}
\vspace{-1ex}
\label{sec:ssr}

As shown in Eq.~\eqref{eq:probability}, diminishing within-group behavioral differences can weaken GRPO's relative safety signals. To effectively optimize the safety objective in Eq.~\eqref{eq:overall_objective}, we recover effective learning signals from two complementary aspects: when relative differences remain, we focus optimization on high-contrast trajectories (Sec.~\ref{sec:agtf}); when relative signals collapse, particularly for uniformly unsafe groups, we introduce explicit safe supervision to provide a clear direction (see Sec.~\ref{sec:srs}).

\subsubsection{Advantage-Guided Trajectory Filtering}
\label{sec:agtf}

Even when relative advantages remain non-zero, the learning utility of different rollouts can vary substantially. In particular, near-zero advantages contribute little to policy updates and may dilute the effective signal from more discriminative trajectories. We therefore introduce advantage-guided trajectory filtering (AGTF), which prioritizes trajectories with strong within-group advantage contrast.

Given $G$ rollouts for a prompt $q$ with normalized advantages $\{\hat{A}_i\}_{i=1}^{G}$, we sort them in descending order as $\hat{A}_{(1)} \geq \cdots \geq \hat{A}_{(G)}$ and construct high--low advantage pairs following a max--min rule:
\vspace{-1ex}
{\setlength{\belowdisplayskip}{2pt}
\begin{equation}
\mathcal{P}_q =
\left\{
\left(y^{(i)}, y^{(G-i+1)}\right)
\right\}_{i=1}^{G/2}.
\end{equation}}

For each pair, we measure its contrast by the advantage gap
$\Delta_i = \hat{A}_{(i)}-\hat{A}_{(G-i+1)}$.
We then rank all pairs by $\Delta_i$ and retain only the top $\kappa$ fraction:
\vspace{-1ex}
{\setlength{\belowdisplayskip}{2pt}
\begin{equation}
\mathcal{P}_q^{*}
=
\operatorname{Top}_{\kappa}
\left(\mathcal{P}_q;\Delta_i\right).
\end{equation}}

Only trajectories in $\mathcal{P}_q^{*}$ are used for subsequent policy optimization. By prioritizing high-contrast trajectories, AGTF reduces the impact of low-information rollouts and focuses the update budget on more informative relative learning signals.

\begin{table*}[t]
  \centering
  \begingroup
  \small
  \setlength{\tabcolsep}{2.5pt}
  \renewcommand{\arraystretch}{1.0}
  \resizebox{0.95\textwidth}{!}{%
  \begin{tabular}{@{}ll*{10}{c}@{}}
  \toprule
  Model & Method
  & \multicolumn{5}{c}{Safety ($\downarrow$)}
  & \multicolumn{5}{c}{Reasoning ($\uparrow$)} \\
  \cmidrule(lr){3-7}\cmidrule(lr){8-12}
  & & HarmBench & StrongREJECT & WildJailbreak & XSTest & Avg.
  & AIME24 & GSM8K & MATH500 & GPQA & Avg. \\
  \midrule
  
  & Initial
  & 38.21 & 8.31 & 43.00 & \textbf{0.80} & 22.58
  & 67.50 & 84.69 & 93.43 & 52.27 & 74.47 \\
  & DirectRefusal
  & 33.06 & 3.19 & 36.20 & 32.00 & 26.11
  & 68.33 & 82.58 & 93.20 & 53.03 & 74.29 \\
  & SafeChain
  & 43.69 & 41.21 & 39.65 & 2.00 & 31.64
  & 62.08 & 89.59 & 93.03 & 51.01 & 73.93 \\
  & STAR-1
  & 33.72 & 5.75 & 35.15 & 6.80 & 20.36
  & 62.50 & 90.97 & 93.05 & 51.96 & 74.62 \\
  & SafePath
  & 37.71 & 7.35 & 42.45 & 1.60 & 22.28
  & 72.08 & 84.45 & 93.33 & 53.54 & 75.85 \\
  & SafeKey
  & 32.39 & 3.19 & 32.95 & 0.80 & 17.33
  & 67.08 & \underline{91.79} & 92.87 & 51.83 & 75.89 \\
  & GRPO
  & 32.56 & 31.95 & 34.45 & 0.80 & 24.94
  & 71.25 & 84.70 & 93.20 & 53.66 & 75.70 \\
  & ThinkSafe
  & \underline{9.63} & \underline{0.32} & \underline{7.45}
  & 2.80 & \underline{5.05}
  & \textbf{73.33} & 88.06 & \underline{93.53}
  & \textbf{53.79} & \underline{77.18} \\
  \rowcolor{oursbg}
  \cellcolor{white}\multirow{-9}{*}{Qwen3-4B}
  & \textbf{Beacon (Ours)}
  & \textbf{3.16} & \textbf{0.00} & \textbf{2.40}
  & \underline{1.20} & \textbf{1.69}\deltadown{20.89}
  & \underline{72.47} & \textbf{91.96} & \textbf{93.87}
  & \underline{53.77} & \textbf{78.02}\deltaup{3.55} \\
  \midrule
  
  & Initial
  & 35.05 & 4.47 & 38.35 & \textbf{0.40} & 19.57
  & 74.17 & 85.28 & \textbf{94.18} & 50.69 & 76.08 \\
  & DirectRefusal
  & 24.42 & 1.92 & 28.05 & 37.60 & 23.00
  & \underline{74.58} & 84.31 & 93.63 & 59.41 & 77.98 \\
  & SafeChain
  & 41.20 & 38.95 & 36.42 & 1.20 & 29.44
  & 70.00 & \textbf{92.98} & 93.53 & 58.21 & 78.68 \\
  & STAR-1
  & 24.42 & \underline{1.28} & 29.25 & 6.80 & 15.44
  & 72.50 & 90.29 & 93.73 & 57.83 & 78.59 \\
  & SafePath
  & 35.22 & 6.71 & 39.45 & 1.20 & 20.64
  & 74.58 & 84.89 & \underline{93.85}
  & \underline{61.24} & 78.64 \\
  & SafeKey
  & 26.91 & 4.79 & 28.80 & 8.80 & 17.33
  & 70.00 & 92.44 & 93.30 & 59.91 & \underline{78.91} \\
  & GRPO
  & 27.74 & 2.24 & 30.10 & 2.00 & 15.52
  & 73.33 & 85.01 & 93.63 & 60.29 & 78.07 \\
  & ThinkSafe
  & \underline{9.14} & 0.32 & \underline{7.35}
  & 1.20 & \underline{4.50}
  & 72.92 & 88.00 & 93.10 & 59.67 & 78.42 \\
  \rowcolor{oursbg}
  \cellcolor{white}\multirow{-9}{*}{Qwen3-8B}
  & \textbf{Beacon (Ours)}
  & \textbf{3.99} & \textbf{0.32} & \textbf{2.00}
  & \underline{1.20} & \textbf{1.88}\deltadown{17.69}
  & \textbf{74.75} & \underline{92.79} & 93.68
  & \textbf{61.34} & \textbf{80.64}\deltaup{4.56} \\
  \midrule
  
  & Initial
  & 56.98 & 63.58 & 53.15 & 1.20 & 43.73
  & 49.58 & 90.32 & 90.18 & 46.65 & 69.18 \\
  & DirectRefusal
  & 52.33 & 33.55 & 50.20 & 43.60 & 44.92
  & 47.50 & 88.27 & 89.82 & 44.95 & 67.64 \\
  & SafeChain
  & 51.00 & 54.63 & 45.85 & 0.40 & 37.97
  & 49.17 & 89.75 & 91.50 & 46.78 & 69.30 \\
  & STAR-1
  & 52.99 & 47.92 & 48.75 & 2.40 & 38.02
  & 45.83 & 90.32 & 90.58 & 46.02 & 68.19 \\
  & SafePath
  & 55.15 & 64.86 & 52.65 & \textbf{0.00} & 43.16
  & \textbf{52.08} & 89.71 & 90.62 & 46.15 & \underline{69.64} \\
  & SafeKey
  & 45.35 & 33.87 & 45.75 & 7.20 & 33.04
  & 43.76 & \underline{90.58} & 89.90 & \underline{47.29} & 67.88 \\
  & GRPO
  & \underline{31.89} & \underline{23.96} & 36.50
  & 6.40 & \underline{24.69}
  & 49.17 & 89.34 & 87.28 & 44.32 & 67.53 \\
  & ThinkSafe
  & 40.20 & 41.85 & \underline{35.40} & 0.40 & 29.46
  & 51.25 & 90.10 & \underline{91.90} & 45.20 & 69.61 \\
  \rowcolor{oursbg}
  \cellcolor{white}\multirow{-9}{*}{\shortstack[l]{DeepSeek-R1\\Distill-7B}}
  & \textbf{Beacon (Ours)}
  & \textbf{30.90} & \textbf{16.93} & \textbf{16.80}
  & \underline{0.40} & \textbf{16.26}\deltadown{27.47}
  & \underline{51.83} & \textbf{90.76} & \textbf{92.06}
  & \textbf{47.33} & \textbf{70.50}\deltaup{1.32} \\
  \midrule
  
  & Initial
  & 52.33 & 53.99 & 49.70 & 0.40 & 39.10
  & 47.50 & \textbf{87.74} & 87.38 & \textbf{48.11} & \underline{67.68} \\
  & DirectRefusal
  & 32.39 & 0.64 & 32.60 & 50.00 & 28.91
  & 40.00 & 83.26 & 85.00 & 43.50 & 62.94 \\
  & SafeChain
  & 44.52 & 46.33 & 42.45 & \underline{1.60} & 33.72
  & 41.67 & 86.06 & 86.50 & 42.05 & 64.07 \\
  & STAR-1
  & 21.26 & \underline{3.51} & 17.60 & 12.00 & \underline{13.59}
  & 40.42 & 87.28 & 86.65 & 43.69 & 64.51 \\
  & SafePath
  & 47.51 & 51.44 & 50.20 & 0.40 & 37.39
  & 43.33 & 87.45 & 87.43 & 47.41 & 66.41 \\
  & SafeKey
  & 32.72 & 11.82 & 30.85 & 8.00 & 20.85
  & 35.83 & 87.41 & 85.80 & 42.49 & 62.88 \\
  & GRPO
  & \textbf{18.77} & 7.70 & \underline{12.40} & 25.60 & 16.12
  & 37.92 & 85.40 & 85.10 & 44.89 & 63.33 \\
  & ThinkSafe
  & 27.08 & 26.52 & 21.15 & 1.60 & 19.09
  & \textbf{48.75} & 87.55 & \underline{87.70}
  & 46.28 & 67.57 \\
  \rowcolor{oursbg}
  \cellcolor{white}\multirow{-9}{*}{\shortstack[l]{DeepSeek-R1\\Distill-8B}}
  & \textbf{Beacon (Ours)}
  & \underline{20.93} & \textbf{0.64} & \textbf{3.10}
  & \textbf{0.40} & \textbf{6.27}\deltadown{32.83}
  & \underline{48.50} & \underline{87.69} & \textbf{87.98}
  & \underline{47.58} & \textbf{67.94}\deltaup{0.26} \\
  \bottomrule
  \end{tabular}%
  }
    \vspace{-2ex}
  \caption{Safety and reasoning performance across LRMs. Safety is measured by harmful response rates (\%, $\downarrow$) on three benchmarks and over-refusal on XSTest. Reasoning is measured by pass@1 (\%, $\uparrow$) on four benchmarks.}
  \label{tab:main-results}
    \vspace{-5ex}
  \endgroup
  \end{table*}
\subsubsection{Safe Reference Supervision}

\label{sec:srs}

While AGTF strengthens informative relative differences when they exist, it cannot provide a correct safety direction when the sampled group itself is uniformly unsafe. For difficult harmful prompts, identical unsafe rewards lead to vanishing advantages; even when reward differences remain, some unsafe responses may still receive positive relative advantages. In either case, relative feedback only distinguishes among sampled responses, without specifying what a safe response should look like.

To address this limitation, we introduce Safe Reference Supervision (SRS), providing an explicit positive target for harmful prompts. For $q\in\mathcal{D}_h$, let $y_q^{*}=(y_1^{*},\ldots,y_T^{*})$ denote its safe reference response. We maximize the autoregressive log-likelihood of the reference response, including the reasoning trace, final answer, and EOS token:
\vspace{-2ex}
{\setlength{\belowdisplayskip}{2pt}
\begin{equation}
  \mathcal{J}_{\mathrm{SRS}}(\theta)
  =
  \frac{1}{T}
  \sum_{t=1}^{T}
  \log
  \pi_\theta
  \left(
  y_t^{*}\mid q,y_{<t}^{*}
  \right).
  \end{equation}}

The training objective for harmful prompts is
\vspace{-1ex}
{\setlength{\belowdisplayskip}{2pt}
\begin{equation}
  \mathcal{J}_{h}(\theta)
  =
  \mathcal{J}_{\mathrm{GRPO}}(\theta)
  +
  \lambda_{\mathrm{SRS}}
  \mathcal{J}_{\mathrm{SRS}}(\theta).
  \end{equation}}
where $\lambda_{\mathrm{SRS}}$ controls the supervision strength. Unlike group-relative feedback, SRS remains informative when within-group safety contrasts vanish, providing an explicit safe target beyond responses that are only relatively better but are still unsafe.

\vspace{-1ex}
\subsection{Safety-Utility Decoupled Alignment}
\vspace{-1ex}
\label{sec:suda}
Stronger safety optimization may cause over-refusal on benign prompts, incurring a safety tax. To balance safety and utility in Eq.~\eqref{eq:overall_objective}, we adopt category-specific optimization: category-aware reward design (Sec.~\ref{sec:card}) distinguishes harmful and benign objectives, while benign-only KL anchoring (Sec.~\ref{sec:boa}) limits policy drift on benign prompts.

  \begin{table*}[t]
    \centering
    \begingroup
    \small
    \setlength{\tabcolsep}{3pt}
    \renewcommand{\arraystretch}{1.0}
    \resizebox{0.95\textwidth}{!}{%
    \begin{tabular}{@{}l*{10}{c}@{}}
    \toprule
    Method
    & \multicolumn{5}{c}{Safety ($\downarrow$)}
    & \multicolumn{5}{c}{Reasoning ($\uparrow$)} \\
    \cmidrule(lr){2-6}\cmidrule(lr){7-11}
    & HarmBench & StrongREJECT & WildJailbreak & XSTest & Avg.
    & AIME24 & GSM8K & MATH500 & GPQA & Avg. \\
    \midrule
    DPO
    & 31.89 & 2.88 & 34.90 & 2.00 & 17.92
    & \textbf{76.67} & 85.76 & 93.75
    & 59.53 & \underline{78.93} \\
    GRPO
    & 27.74 & 2.24 & 30.10
    & 2.00 & 15.52
    & 73.33 & 85.01 & 93.63 & 60.29 & 78.07 \\
    SFT+GRPO
    & \underline{23.24} & \underline{1.89} & \underline{23.24}
    & \underline{1.60} & \underline{12.49}
    & 73.25 & \underline{86.78} & 93.84 & 60.94 & 78.70 \\
    \midrule
    GDPO
    & 28.74 & 2.56 & 31.50 & 1.60 & 16.10
    & 70.83 & 84.94 & 94.00 & 60.48 & 77.56 \\
    NGRPO
    & 27.08 & 2.56 & 32.30 & 2.40 & 16.09
    & 73.33 & 85.29 & \underline{94.10}
    & \underline{61.17} & 78.47 \\
    MDP-GRPO
    & 28.57 & 3.83 & 31.35 & 2.40 & 16.54
    & 73.75 & 85.47 & \textbf{94.17}
    & 60.98 & 78.59 \\
    \midrule
    \rowcolor{oursbg}
    \textbf{Beacon (Ours)}
    & \textbf{3.99} & \textbf{0.32} & \textbf{2.00}
    & \textbf{1.20} & \textbf{1.88}
    & \underline{74.75} & \textbf{92.79} & 93.68
    & \textbf{61.34} & \textbf{80.64} \\
    \bottomrule
    \end{tabular}%
    }
    \vspace{-2ex}
    \caption{Comparison with optimization methods, including methods addressing advantage collapse. Safety is evaluated on three harmful benchmarks and XSTest (\%, $\downarrow$), and reasoning by pass@1 on four benchmarks (\%, $\uparrow$).}
    \vspace{-3ex}
    \label{tab:optimization-comparison}
    \endgroup
    \end{table*}

\vspace{-1ex}
\subsubsection{Category-Aware Reward Design}
\label{sec:card}

To mitigate over-refusal induced by safety optimization, we define category-specific rewards for harmful and benign prompts. The reward model assigns each response three continuous scores: safety $r_s$, utility $r_u$, and honesty $r_o$. Each score incorporates evaluations of both the reasoning trace and the final answer. Let $\bar r(y)=[r_s(y)+r_u(y)+r_o(y)]/3$ denote their mean. We prioritize safety for harmful prompts, whereas the reward for benign prompts jointly considers all three different dimensions:
\vspace{-2ex}
{\setlength{\belowdisplayskip}{2pt}
\begin{equation}
  \resizebox{0.85\columnwidth}{!}{$
R(q,y)
=
F(y)
\begin{cases}
r_s(y),
& q\in\mathcal{D}_h,\\[1mm]
\displaystyle
\min\left(
r_s(y),\bar r(y)
\right),
& q\in\mathcal{D}_b,
\end{cases}
$}
\end{equation}}
where $F(y)\in\{0,1\}$ is a format-validity gate. For benign prompts, the minimum operator prevents the reward from being dominated by safety alone, avoiding overly conservative responses with high safety but low utility (i.e., over-refusal). This category-aware design prioritizes safety on harmful prompts while rewarding safe and useful responses to benign prompts, thereby effectively mitigating over-refusal caused by a uniform safety objective.

\vspace{-1ex}
\subsubsection{Benign-Only KL Anchoring}
\label{sec:boa}
While category-aware rewards explicitly emphasize utility on benign prompts, continued safety optimization may still shift the policy away from its original benign behavior, resulting in unnecessary refusals. We therefore introduce benign-only KL anchoring, which constrains the policy toward its pre-alignment behavior only on benign prompts.
Specifically, for $q\in\mathcal{D}_b$, we further augment the GRPO objective with an explicit KL constraint:
\vspace{-1ex}
{\setlength{\belowdisplayskip}{2pt}
\begin{equation}
  \resizebox{0.85\columnwidth}{!}{$
\mathcal{J}_{b}(\theta)
=
\mathcal{J}_{\mathrm{GRPO}}(\theta)
-
\lambda_b
D_{\mathrm{KL}}
\left(
\pi_\theta(\cdot|q)
\Vert
\pi_{\mathrm{base}}(\cdot|q)
\right),
$}
\end{equation}}
where $\pi_{\mathrm{base}}$ denotes the policy before safety post-training and $\lambda_b$ controls the regularization strength. By applying this constraint exclusively to benign prompts, we limit unnecessary policy drift on normal requests without restricting the safety updates required for harmful ones. Together with category-aware rewards, this KL anchoring design further mitigates the safety tax caused by over-refusal.

\vspace{-1ex}
\section{Experiments}
\vspace{-1ex}

\begin{figure}[t]
  \centering
  \vspace{-2ex}
  \includegraphics[width=0.49\textwidth]{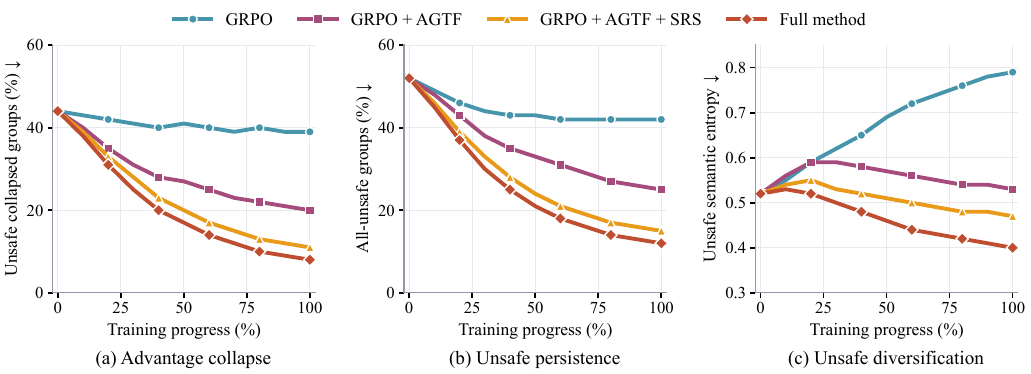}
  \vspace{-4ex}
  \caption{Comparison of GRPO variants throughout training. (a) Advantage collapse: proportion of all-unsafe groups with near-zero advantages. (b) Unsafe persistence: proportion of all-unsafe groups. (c) Unsafe diversification: semantic entropy of unsafe responses.}
  \label{fig:advantage-unsafe-mechanism}
  \vspace{-4ex}
\end{figure}
\subsection{Experimental Setup}
\textbf{Baselines and datasets.} We compare several representative categories of baselines: SFT-based methods, including DirectRefusal~\cite{huang2025safetytaxsafetyalignment}, SafeChain~\cite{jiang-etal-2025-safechain}, STAR-1~\cite{wang2026star}, SafePath~\cite{NEURIPS2025_8fe84813}, SafeKey~\cite{zhou-etal-2025-safekey}, and ThinkSafe~\cite{lee2026thinksafe}; RL-based and hybrid optimization methods, including GRPO~\cite{shao2024deepseekmathpushinglimitsmathematical}, DPO~\cite{NEURIPS2023_a85b405e}, and SFT+GRPO; and advantage-collapse mitigation methods, including NGRPO~\cite{nan2025ngrponegativeenhancedgrouprelative}, GDPO~\cite{liu2026gdpogrouprewarddecouplednormalization}, and MDP-GRPO~\cite{salmani-zarchi-etal-2026-mdp}. We evaluate reasoning on GSM8K~\cite{cobbe2021trainingverifierssolvemath}, MATH500~\cite{ICLR2024_aca97732}, AIME24~\cite{aime24}, and GPQA~\cite{rein2023gpqagraduatelevelgoogleproofqa}, averaging correctness over eight generations for pass@1. Safety is evaluated on StrongREJECT~\cite{NEURIPS2024_e2e06adf}, HarmBench~\cite{mazeika2024harmbenchstandardizedevaluationframework}, and WildJailbreak~\cite{NEURIPS2024_54024fca}, reporting harmful response rates from single generations judged by Llama-Guard-3~\cite{dubey2024llama3herdmodels}. Over-refusal is measured on the safe subset of XSTest~\cite{rottger-etal-2024-xstest} using WildGuard~\cite{NEURIPS2024_0f69b4b9}. Further details are provided in Appendix~\ref{app:evaluation_details}.

  \begin{figure}[t]
  \centering
  \vspace{-2ex}
  \includegraphics[width=0.49\textwidth]{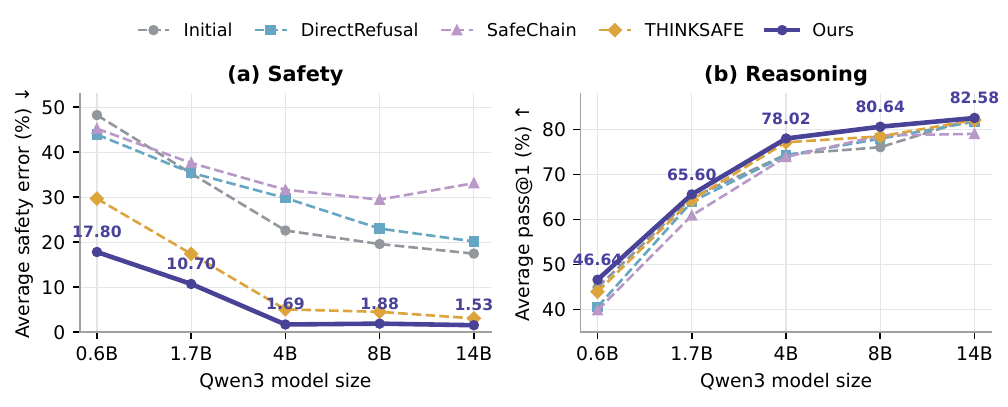}
  \vspace{-5ex}
  \caption{Performance comparison across Qwen3 model scales. (a) Safety: average safety error across four benchmarks (lower is better). (b) Reasoning: average pass@1 across four benchmarks (higher is better).}
  \label{fig:scaling}
  \vspace{-4ex}
  \end{figure}
  
\noindent\textbf{Implementation details.} We train models from the DeepSeek-R1-Distill and Qwen3 (0.6B--14B) families using the training data from ThinkSafe, with its safe reference responses serving as supervision targets for SRS. We set both $\lambda_{\mathrm{SRS}}$ and $\lambda_b$ to $1$, and select the trajectory retention ratio $\kappa=0.75$ based on ablation studies. We use ArmoRM-Llama3-8B \cite{ArmoRM}, which ranked first among reward models on RewardBench \cite{lambert-etal-2025-rewardbench}, to compute the safety, utility, and honesty scores. Further details are provided in Appendix~\ref{app:training Details}.

\vspace{-1ex}
\subsection{Main Results}
\vspace{-1ex}

\textbf{Overall performance.}
We evaluate whether our framework improves LRM safety while preserving reasoning capabilities and limiting over-refusal. As shown in Tab.~\ref{tab:main-results}, our method achieves the lowest average safety error and the highest average reasoning accuracy across all four models. Compared with the initial policies, our method reduces the average safety error by 17.69--32.83 percentage points and improves reasoning accuracy by 0.26--4.56 percentage points, while maintaining low over-refusal rates of 0.40\%--1.20\% on XSTest. These results indicate that improved safety does not come at the cost of indiscriminate refusal. Notably, on DeepSeek-8B, GRPO achieves a safety error of 16.12\%, a reasoning accuracy of 63.33\%, and an over-refusal rate of 25.60\%. In contrast, our method achieves a lower safety error of 6.27\%, a higher reasoning accuracy of 67.94\%, and a substantially lower over-refusal rate of 0.40\%. These results demonstrate the effectiveness of combining safety signal recovery with safety-utility decoupled alignment to strengthen safety while mitigating the resulting safety tax.

  \begin{figure}[t]
  \centering
  \includegraphics[width=\columnwidth]{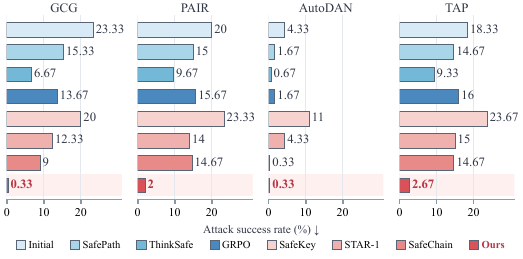}
  \vspace{-5ex}
  \caption{Robustness comparison against four adversarial attacks: GCG, PAIR, AutoDAN, and TAP. Attack success rates (\%, $\downarrow$) are evaluated on HarmBench.}
  \label{fig:adversarial-attacks}
  \vspace{-5ex}
\end{figure}

\noindent\textbf{Comparison with alternative optimization methods.}
We compare our framework with standard optimization methods and approaches developed to mitigate advantage collapse in other settings on Qwen3-8B. As shown in Tab.~\ref{tab:optimization-comparison}, our method achieves the lowest average safety error (1.88\%) and the highest average reasoning accuracy (80.64\%), outperforming GRPO by 13.64 and 2.57 percentage points, respectively. GDPO, NGRPO, and MDP-GRPO yield no improvement over GRPO in average safety error, suggesting that generic treatments of advantage collapse alone are insufficient for safety alignment. These results support the effectiveness of combining safety signal recovery with safety-utility decoupled alignment.

\noindent\textbf{Advantage collapse and unsafe diversification.}
To examine whether our framework mitigates both failure modes, we compare GRPO, GRPO with AGTF, GRPO with AGTF and SRS, and the full framework on Qwen3-8B. We track the proportion of collapsed all-unsafe groups ($\max_i|\hat{A}_i|<0.05$), the overall proportion of all-unsafe groups, and normalized semantic entropy computed from equally sampled unsafe responses. As shown in Fig.~\ref{fig:advantage-unsafe-mechanism}, the successive incorporation of AGTF, SRS, and SUDA generally reduces these metrics, with the full framework achieving the lowest values across all three. These results suggest that the proposed components jointly mitigate advantage collapse and unsafe diversification in safety alignment.
\begin{table}[t]
  \centering
  \footnotesize
  \setlength{\tabcolsep}{1pt}
  \renewcommand{\arraystretch}{1.05}
  
  \begin{tabular}{lcc|cc|cc}
  \toprule
  \multirow{2}{*}{Method}
  & \multicolumn{2}{c|}{Llama-Guard-3}
  & \multicolumn{2}{c|}{GPT-6}
  & \multicolumn{2}{c}{Human}
  \\
  \cmidrule(lr){2-3}
  \cmidrule(lr){4-5}
  \cmidrule(lr){6-7}
  & S$\downarrow$ & R$\uparrow$
  & S$\downarrow$ & R$\uparrow$
  & S$\downarrow$ & R$\uparrow$
  \\
  \midrule
  
  Initial
  & 19.57 & 76.08
  & 20.45 & 77.29
  & 19.78 & 76.46
  \\
  
  DirectRefusal
  & 23.00 & 77.98
  & 25.32 & 78.46
  & 24.31& 80.23
  \\
  
  SafeChain
  & 29.44 & 78.68
  & 30.13 & 79.69
  & 28.49 & 80.31
  \\
  
  STAR-1
  & 15.44 & 78.59
  & 17.31 & 79.45
  & 16.54 & 78.49
  \\
  
  SafePath
  & 20.64 & 78.64
  & 19.48 & 79.43
  & 18.57 & 77.88
  \\
  
  SafeKey
  & 17.33 & 78.91
  & 16.98 & 79.46
  & 16.58 & 78.57
  \\
  
  GRPO
  & 15.52 & 78.07
  & 16.45 & 77.94
  & 16.78 & 79.41
  \\
  
  ThinkSafe
  & 4.50 & 78.42
  & 5.31 & 79.56
  & 4.89 & 79.91
  \\
  
  \textbf{Beacon (Ours)}
  & \textbf{1.88} & \textbf{80.64}
  & \textbf{2.10} & \textbf{81.46}
  & \textbf{1.94} & \textbf{80.73}
  \\
  
  \bottomrule
  \end{tabular}
  \vspace{-2ex}
  \caption{Evaluator consistency on Qwen3-8B. Safety includes harmful response rates and over-refusal
  (\%, $\downarrow$), while reasoning is measured by pass@1 (\%, $\uparrow$).}
      \vspace{-4ex}
  \label{tab:evaluator-consistency}
  \end{table}
\vspace{-1ex}
\subsection{Additional Analysis}
\vspace{-1ex}
  
\noindent\textbf{Robustness across model scales.}
  To evaluate the robustness of our framework across model scales, we conduct experiments on Qwen3 models ranging from 0.6B to 14B parameters. As shown in Fig.~\ref{fig:scaling}, our method consistently achieves the lowest average safety error and the highest average reasoning accuracy across all five model sizes. Compared with the initial models, it reduces average safety error by 15.56--30.42 percentage points while maintaining or improving reasoning performance. These results suggest that the benefits of safety signal recovery and safety-utility decoupled alignment persist across model scales, enabling reliable safety gains without sacrificing reasoning capabilities.

\noindent\textbf{Robustness against adversarial attacks.}
We evaluate whether our framework maintains its safety alignment under adversarial attacks~\cite{10.1145/3664647.3680616,chen2026tex3dobjectsattacksurfaces}. Following LPA's evaluation setup~\cite{merzouk2026efficient}, we test it on 300 harmful behaviors from HarmBench using GCG, PAIR, AutoDAN, and TAP. As shown in Fig.~\ref{fig:adversarial-attacks}, our method achieves the lowest attack success rates under GCG, PAIR, and TAP, and ties with SafeChain under AutoDAN. Its average attack success rate is 1.33\%, outperforming ThinkSafe and GRPO by 5.25 and 10.42 percentage points, respectively. These results suggest that the safety improvements extend to adversarially optimized prompts, supporting the robustness of our alignment framework across different attack strategies~\cite{chen-etal-2025-autobreach}.

\noindent\textbf{Evaluator Consistency.}
To examine the robustness of our results across evaluators, we assess Qwen3-8B using Llama-Guard-3, GPT-6, and human annotations. As shown in Table~\ref{tab:evaluator-consistency}, our method consistently achieves the lowest safety and highest reasoning scores across all three evaluators. Despite minor variations in absolute scores, the improvements remain consistent, suggesting that our results do not depend on a particular evaluator.
\vspace{-1ex}
\subsection{Ablation Studies}
\vspace{-1ex}
\noindent\textbf{Ablation of key components.} To evaluate the contribution of each component, we progressively add AGTF, SRS, CAR, and BKL on Qwen3-8B. Safe Rollout measures the percentage of safe responses to harmful prompts. As shown in Tab.~\ref{tab:component-ablation}, SRS increases Safe Rollout from 24.34\% to 49.74\% and reduces average safety error from 5.16\% to 2.36\%, indicating that safe reference supervision compensates for insufficient relative feedback. Adding CAR and BKL reduces XSTest over-refusal from 14.40\% to 1.20\% and improves reasoning accuracy from 78.89\% to 80.64\%. These results highlight the roles of SSR in recovering safety signals and SUDA in preserving utility during alignment.

\begin{table}[t]
  \centering
  \setlength{\tabcolsep}{3pt}
  \renewcommand{\arraystretch}{1.12}
  \resizebox{0.95\columnwidth}{!}{
  \begin{tabular}{cc|cc|cccc}
    \toprule
    \multicolumn{2}{c|}{SSR}
    & \multicolumn{2}{c|}{SUDA}
    & \multirow{2}{*}{\rotatebox{45}{Safety~($\downarrow$)}}
    & \multirow{2}{*}{\rotatebox{45}{XSTest~($\downarrow$)}}
    & \multirow{2}{*}{\rotatebox{45}{Reasoning~($\uparrow$)}}
    & \multirow{2}{*}{\rotatebox{45}{Safe Rollout (\%)}} \\
    \cmidrule(lr){1-2}\cmidrule(lr){3-4}
    \raisebox{12pt}{AGTF}
    & \raisebox{12pt}{SRS}
    & \raisebox{12pt}{CAR}
    & \raisebox{12pt}{BKL}
    & \rule{0pt}{35pt} & & & \\
    \midrule
    \ding{51} & --        & --        & --
    & 5.16 & 10.20 & 78.26 & 24.34 \\
    \ding{51} & \ding{51} & --        & --
    & 2.36 & 14.40 & 78.89 & 49.74 \\
    --        & \ding{51} & \ding{51} & \ding{51}
    & 5.42& 1.40 & 79.21& 50.46\\
    \ding{51} & --        & \ding{51} & \ding{51}
    & 4.68& 1.60& 79.45&50.38 \\
    \ding{51} & \ding{51} & --        & \ding{51}
    & 2.04& 7.80& 79.96&50.32 \\
    \ding{51} & \ding{51} & \ding{51} & --
    & 2.25 & 6.60 & 79.94 & 50.34 \\
    \rowcolor{MethodTint}
    \ding{51} & \ding{51} & \ding{51} & \ding{51}
    & 1.88 & 1.20 & 80.64 & 50.64 \\
    \bottomrule
  \end{tabular}
  }
  \vspace{-2ex}
  \caption{Ablation study of the four key components of our framework: AGTF, SRS, CAR, and BKL.}
  \vspace{-2ex}
  \label{tab:component-ablation}
\end{table}

\noindent\textbf{Sensitivity to the trajectory retention ratio.}
We vary $\kappa\in\{0.25,0.50,0.75,1.00\}$ on Qwen3 (1.7B-8B), keeping other training settings fixed. Fig.~\ref{fig:kappa-ablation} shows performance changes relative to $\kappa=1.00$, which retains all trajectories. Moderate filtering improves safety and reasoning, with $\kappa=0.75$ achieving the best results across all three models. In contrast, retaining only 25\% of trajectories generally degrades performance. These findings highlight a trade-off between prioritizing informative relative signals and retaining sufficient training samples.

\vspace{-1ex}
\section{Related Work}
\vspace{-1ex}
\subsection{Safety Alignment of LRMs}

Safety alignment of LRMs aims to reduce harmful outputs while preserving their ability to respond effectively to benign and legitimate user requests~\cite{fu2026reducingsafetytaxllm}. Existing approaches mainly fall into two categories: supervised fine-tuning and reinforcement-learning-based safety alignment. The former relies on high-quality safety data or safety reasoning trajectories~\cite{wang2026star,zhou-etal-2025-safekey,lee2026thinksafe}. The latter directly optimizes model behavior through preference-based objectives~\cite{NEURIPS2023_a85b405e,ICLR2026_57aabe6f} or reinforcement-learning objectives~\cite{kim2026reasoning,dineen-etal-2025-qa}. Although these methods improve the safety reasoning of LRMs, persistent unsafe responses on difficult prompts remain insufficiently addressed.
\begin{figure}[t]
  \centering
  \includegraphics[width=\columnwidth]{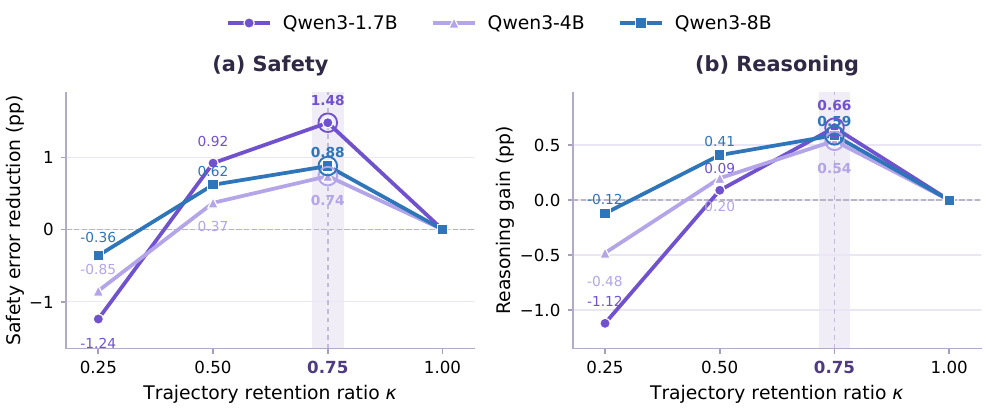}
  \vspace{-5ex}
  \caption{Sensitivity analysis of trajectory retention ratio $\kappa$ across three Qwen3 model scales. We report changes in safety error (\%, $\uparrow$) and reasoning accuracy (\%, $\uparrow$) relative to $\kappa=1.00$, which retains all trajectories.}
  \vspace{-4ex}
  \label{fig:kappa-ablation}
\end{figure}

\vspace{-1ex}
\subsection{Optimization of GRPO}
\vspace{-1ex}

GRPO is widely used for LRM post-training because it does not require an additional value model. However, it may suffer from advantage collapse. Existing methods address this issue by reconstructing advantage signals or introducing virtual samples~\cite{nan2025ngrponegativeenhancedgrouprelative,he2026advantage,ICLR2026_c64dfa06}, and by improving reward normalization or incorporating external guidance~\cite{liu2026gdpogrouprewarddecouplednormalization,zhang2025edgegrpoentropydrivengrpoguided,ICLR2026_d604a371}. These methods mainly focus on whether sampled responses provide useful relative feedback. However, their effectiveness in safety alignment, where optimization must improve harmlessness without compromising helpfulness, remains underexplored. We therefore analyze the optimization dynamics of GRPO from a safety perspective and propose a GRPO-based framework for LRM safety alignment guided by these findings.
\vspace{-1ex}
\section{Conclusion}
\vspace{-1ex}
In this paper, we investigate the optimization challenges of GRPO-based safety alignment for large reasoning models. We identify two distinct failure modes of vanilla GRPO, namely advantage collapse and unsafe diversification, and provide theoretical insights into how vanishing relative safety signals limit effective optimization. Based on this analysis, we propose SSR and SUDA to recover informative safety signals while mitigating the safety tax. Extensive experiments across multiple LRM variants and benchmarks demonstrate that our framework substantially improves safety alignment while maintaining low over-refusal. Our findings highlight the importance of designing safety-aware optimization mechanisms that account for the unique dynamics of large reasoning models.

\section{Limitations}
Our work has several limitations that motivate future research. First, Beacon relies on existing reward models and safety evaluators to provide safety, utility, and honesty signals. Although we observe consistent improvements across multiple evaluators, the effectiveness of the framework may still depend on the quality and reliability of these signals. Developing more robust and adaptive reward modeling approaches is an important direction for future work. Second, this work primarily focuses on safety alignment under harmful and benign prompt settings. Extending Beacon to broader alignment objectives, such as factuality, uncertainty calibration, and long-horizon interactions, could further improve its applicability to real-world alignment scenarios.

\bibliography{custom}        
\clearpage
\appendix

\section{Proofs of Theoretical Results}
\label{app:theoretical-proofs}

This appendix provides the complete proofs of
Theorems~\ref{thm:vanishing_signal} and~\ref{thm:unsafe_flattening}. For simplicity,
$\pi_\theta$ denotes the behavior policy used to generate the rollout
group at the current optimization step.
For a prompt $q$, let $G$ responses be sampled i.i.d. from
$\pi_\theta(\cdot\mid q)$, and let $S(y)\in\{0,1\}$ indicate whether
a response $y$ is safe.

\subsection{Proof of Theorem~\ref{thm:vanishing_signal}: 
Vanishing Relative Safety Signal}
\label{app:proof_theorem1}

\begin{proof}
For a fixed prompt $q$, define
\begin{equation}
    p_q =
    \Pr_{y\sim\pi_\theta(\cdot\mid q)}
    [S(y)=1].
\end{equation}
A direct relative safety contrast is available only when the sampled
group contains at least one safe and at least one unsafe response.
The complementary events are that all $G$ responses are safe or that
all $G$ responses are unsafe. Under the i.i.d. sampling assumption,
these events occur with probabilities $p_q^G$ and $(1-p_q)^G$,
respectively. Hence,
\begin{equation}
    P_{\mathrm{mix}}(q)
    =
    1-p_q^G-(1-p_q)^G.
\end{equation}

Taking the two limiting regimes gives
\begin{equation}
p_q \to 0 \ \text{or}\ p_q \to 1
\Longrightarrow P_{\mathrm{mix}}(q) \to 0.
\end{equation}
Therefore, as the policy becomes behaviorally consistent, groups that
cross the safety boundary become increasingly unlikely.

For GRPO, the advantage of the $i$-th response is normalized within
the sampled group as
\begin{equation}
    \hat{A}_i
    =
    \frac{r_i-\mu_R}{\sigma_R+\epsilon},
\end{equation}
where $\mu_R$ and $\sigma_R$ denote the mean and standard deviation of
the group rewards, respectively.

For a homogeneous group with identical rewards,
\begin{equation}
    r_1=\cdots=r_G=r.
\end{equation}
It follows that
\begin{equation}
    \mu_R=r,
    \qquad
    \sigma_R=0.
\end{equation}
Therefore, for every rollout $i$,
\begin{equation}
    \hat{A}_i
    =
    \frac{r-\mu_R}{\sigma_R+\epsilon}
    =
    \frac{0}{\epsilon}
    =
    0.
\end{equation}
Hence,
\begin{equation}
    \hat{A}_1=\cdots=\hat{A}_G=0.
\end{equation}
Thus, GRPO receives no direct within-group safety contrast in either
the all-safe or all-unsafe regime. This completes the proof of
Theorem~\ref{thm:vanishing_signal}.
\end{proof}

\subsection{Proof of Theorem~\ref{thm:unsafe_flattening}: 
Unsafe Distribution Flattening}
\label{app:proof_theorem2}

\begin{proof}
Consider a difficult prompt $q$ and the locally reachable response set
$\mathcal{U}_q$ defined in
Theorem~\ref{thm:unsafe_flattening}. Suppose that all responses in
$\mathcal{U}_q$ are unsafe and receive identical rewards. Their
group-relative advantages therefore vanish, so the relative policy term
provides no local learning signal. Following Eq.~\eqref{eq:reg_objective}, the remaining
optimization objective is
\begin{equation}
    \mathcal{J}_{\mathrm{reg}}(\theta)
    =
    -\lambda_{\mathrm{KL}}
    D_{\mathrm{KL}}
    \bigl(
        \pi_\theta \,\|\, \pi_{\mathrm{ref}}
    \bigr)
    +
    \lambda_H \mathcal{H}(\pi_\theta).
\end{equation}

Since Theorem~\ref{thm:unsafe_flattening} considers optimization over
the locally reachable set $\mathcal{U}_q$, the following derivation is
restricted to $\mathcal{U}_q$. Accordingly, all policy distributions are understood
as conditional distributions restricted to $\mathcal{U}_q$ and
renormalized within this set. To avoid introducing additional notation,
we continue to use $\pi_{\mathrm{ref}}$ to denote the normalized
reference distribution induced on $\mathcal{U}_q$. We further assume
that $\pi_{\mathrm{ref}}(y\mid q)>0$ for all
$y\in\mathcal{U}_q$.

For notational convenience, define
\begin{equation}
    p_y = \pi_\theta(y\mid q),
    r_y = \pi_{\mathrm{ref}}(y\mid q),
    y\in\mathcal{U}_q.
\end{equation}
Then
\begin{equation}
    \sum_{y\in\mathcal{U}_q} p_y = 1.
\end{equation}
Expanding the KL divergence and entropy terms in Eq.~\eqref{eq:reg_objective} gives
\begin{align}
    \mathcal{J}_{\mathrm{reg}}(\theta)
    &=
    -\lambda_{\mathrm{KL}}
    \sum_{y\in\mathcal{U}_q}
    p_y
    \log\frac{p_y}{r_y}
    \nonumber\\
    &\quad
    -
    \lambda_H
    \sum_{y\in\mathcal{U}_q}
    p_y\log p_y
    \\
    &=
    \lambda_{\mathrm{KL}}
    \sum_{y\in\mathcal{U}_q}
    p_y\log r_y
    \nonumber\\
    &\quad
    -
    (\lambda_{\mathrm{KL}}+\lambda_H)
    \sum_{y\in\mathcal{U}_q}
    p_y\log p_y.
\end{align}

To maximize this objective under the normalization constraint, we introduce
a Lagrange multiplier $\mu$:
\begin{equation}
    \mathcal{L}
    =
    \mathcal{J}_{\mathrm{reg}}(\theta)
    +
    \mu
    \left(
        \sum_{y\in\mathcal{U}_q} p_y - 1
    \right).
\end{equation}
For each $y\in\mathcal{U}_q$, the stationarity condition is
\begin{equation}
    \frac{\partial\mathcal{L}}{\partial p_y}
    =
    \lambda_{\mathrm{KL}}\log r_y
    -
    (\lambda_{\mathrm{KL}}+\lambda_H)
    (\log p_y + 1)
    +
    \mu
    =
    0.
\end{equation}
Hence,
\begin{equation}
    \log p_y
    =
    \frac{\lambda_{\mathrm{KL}}}
    {\lambda_{\mathrm{KL}}+\lambda_H}
    \log r_y
    +
    C,
\end{equation}
where $C$ does not depend on $y$. Define
\begin{equation}
    \alpha
    =
    \frac{\lambda_{\mathrm{KL}}}
    {\lambda_{\mathrm{KL}}+\lambda_H}.
\end{equation}
Since $\lambda_{\mathrm{KL}}+\lambda_H>0$, the objective is strictly
concave over the probability simplex. Therefore, the stationary point is the unique global maximizer.

Exponentiating and normalizing over $\mathcal{U}_q$ gives
\begin{equation}
    \pi^*(y\mid q)
    =
    \frac{
        \pi_{\mathrm{ref}}(y\mid q)^\alpha
    }{
        \sum_{y'\in\mathcal{U}_q}
        \pi_{\mathrm{ref}}(y'\mid q)^\alpha
    },
\end{equation}
which recovers Eq.~\eqref{eq:optimal_flattening}.

When $\lambda_{\mathrm{KL}}>0$ and $\lambda_H>0$, we have
$0<\alpha<1$. For any two responses
$y_i,y_j\in\mathcal{U}_q$,
\begin{equation}
    \frac{
        \pi^*(y_i\mid q)
    }{
        \pi^*(y_j\mid q)
    }
    =
    \left(
        \frac{
            \pi_{\mathrm{ref}}(y_i\mid q)
        }{
            \pi_{\mathrm{ref}}(y_j\mid q)
        }
    \right)^\alpha.
\end{equation}
Because $0<\alpha<1$, the transformation compresses probability ratios
toward one. Therefore, high-probability responses are relatively
down-weighted, while low-probability responses receive relatively more
probability mass, making $\pi^*$ flatter than $\pi_{\mathrm{ref}}$ over
$\mathcal{U}_q$.

It remains to verify the entropy relation in
Eq.~\eqref{eq:entropy_flattening}. Define
\begin{equation}
    \pi_\alpha(y)
    =
    \frac{
        \pi_{\mathrm{ref}}(y)^\alpha
    }{
        Z(\alpha)
    },
    Z(\alpha)
    =
    \sum_{y\in\mathcal{U}_q}
    \pi_{\mathrm{ref}}(y)^\alpha.
\end{equation}
Its entropy is
\begin{equation}
    \mathcal{H}(\pi_\alpha)
    =
    -\alpha
    \mathbb{E}_{\pi_\alpha}
    \left[
        \log\pi_{\mathrm{ref}}(y)
    \right]
    +
    \log Z(\alpha).
\end{equation}
Differentiating with respect to $\alpha$ yields
\begin{equation}
    \frac{
        d\mathcal{H}(\pi_\alpha)
    }{
        d\alpha
    }
    =
    -\alpha\,
    \mathrm{Var}_{\pi_\alpha}
    \left[
        \log\pi_{\mathrm{ref}}(y)
    \right]
    \le 0.
\end{equation}
Thus, $\mathcal{H}(\pi_\alpha)$ is non-increasing in $\alpha$.
Since $\pi_1=\pi_{\mathrm{ref}}$ and $0<\alpha<1$, we obtain
\begin{equation}
    \mathcal{H}(\pi^*)
    \ge
    \mathcal{H}(\pi_{\mathrm{ref}}),
\end{equation}
which is Eq.~\eqref{eq:entropy_flattening}.

Therefore, when the relative safety signal vanishes, the entropy term
flattens the local distribution over unsafe responses, while the KL term
limits deviation from the reference policy. This completes the proof of
Theorem~\ref{thm:unsafe_flattening}.
\end{proof}

\section{Evaluation Details}
\label{app:evaluation_details}

We follow the SkyThought~\cite{sky_t1_2025} evaluation protocol for
AIME 2024, GSM8K, MATH500, and GPQA. Each example is decoded 8 times
with its benchmark-specific prompt, and performance is measured by the
average pass@1. Qwen models are evaluated with temperature $0.6$, top-$p$ $0.95$,
top-$k$ $20$, and a $32{,}768$-token generation cap. For
DeepSeek-R1-Distill models, we use temperature $0.6$ and top-$p$ $0.95$
under the same length constraint. The exact prompting formats are given
below.

\begin{tcolorbox}[
    colback=promptbg,
    colframe=promptframe,
    colbacktitle=prompttitle,
    coltitle=black,
    fonttitle=\bfseries,
    title=AIME 2024,
    boxrule=0.7pt,
    arc=1mm,
    left=2mm,
    right=2mm,
    top=1.2mm,
    bottom=1.2mm,
    before skip=4pt,
    after skip=4pt
]
Please reason step by step, and put your final answer within
\texttt{\textbackslash\textbackslash boxed\{\{\}\}}. \texttt{\{problem\}}
\end{tcolorbox}

\begin{tcolorbox}[
    colback=promptbg,
    colframe=promptframe,
    colbacktitle=prompttitle,
    coltitle=black,
    fonttitle=\bfseries,
    title=GSM8K,
    boxrule=0.7pt,
    arc=1mm,
    left=2mm,
    right=2mm,
    top=1.2mm,
    bottom=1.2mm,
    before skip=4pt,
    after skip=4pt
]
Given the following problem, reason and give a final answer to the problem.

\medskip
\texttt{Problem: \{question\}}

\medskip
Your response should end with ``The final answer is [answer]'' where
\texttt{[answer]} is the response to the problem.
\end{tcolorbox}

\begin{tcolorbox}[
    colback=promptbg,
    colframe=promptframe,
    colbacktitle=prompttitle,
    coltitle=black,
    fonttitle=\bfseries,
    title=MATH500,
    boxrule=0.7pt,
    arc=1mm,
    left=2mm,
    right=2mm,
    top=1.2mm,
    bottom=1.2mm,
    before skip=4pt,
    after skip=4pt
]
Please reason step by step, and put your final answer within
\texttt{\textbackslash\textbackslash boxed\{\{\}\}}. \texttt{\{problem\}}
\end{tcolorbox}

\begin{tcolorbox}[
    colback=promptbg,
    colframe=promptframe,
    colbacktitle=prompttitle,
    coltitle=black,
    fonttitle=\bfseries,
    title=GPQA,
    boxrule=0.7pt,
    arc=1mm,
    left=2mm,
    right=2mm,
    top=1.2mm,
    bottom=1.2mm,
    before skip=4pt,
    after skip=4pt
]
Return your final response within
\texttt{\textbackslash\textbackslash boxed\{\{\}\}} and only include the letter choice
(A, B, C, or D) as your final response. \texttt{\{problem\}}
\end{tcolorbox}

\section{Training Details}
\label{app:training Details}

We fine-tune all models with LoRA, using rank $r=32$, scaling factor
$\alpha=16$, and a dropout rate of 0.05. LoRA is applied only to the
\texttt{q\_proj} and \texttt{v\_proj} modules in the attention layers.
We use AdamW with an initial learning rate of $1\times10^{-5}$ and a
cosine learning-rate schedule with linear warmup over the first 10\% of
training. All models are trained for 4,500 optimization steps using BF16
precision with gradient checkpointing enabled. For each prompt, the policy samples 8 responses, and each rollout
batch contains nine prompts, resulting in 72 responses per batch.
During rollout generation, we set the temperature to 0.7, top-$p$ to
1.0, the maximum prompt length to 1,024 tokens, and the maximum
completion length to 8,192 tokens. The weights for SSR and Benign-Only
KL regularization are both set to 1.0. For the base GRPO objective, the entropy regularization coefficient
$\lambda_H$ and the KL coefficient $\lambda_{\mathrm{KL}}$ are also both set to 1.0. Experiments with models ranging from 0.6B to 8B parameters are conducted
on 4 NVIDIA A100 GPUs, while experiments with 14B models use 8 NVIDIA A100 GPUs.

\section{Training Algorithm}
\label{app:training-algorithm}

Algorithm~\ref{alg:beacon} summarizes one optimization step of Beacon.
It combines category-aware rewards and advantage-guided trajectory filtering
(AGTF) for every prompt with the category-specific objectives of safe reference
supervision (SRS) and benign-only KL anchoring.

\begin{algorithm}[H]
\caption{One training step of Beacon}
\label{alg:beacon}
\begin{algorithmic}[1]
\Require Policy $\pi_\theta$, pre-alignment policy $\pi_{\mathrm{base}}$,
  minibatch $\mathcal B\subset\mathcal D_h\cup\mathcal D_b$, safe references
  $y_q^*$ for $q\in\mathcal D_h$, group size $G$, retention ratio $\kappa$
\State $\theta_{\mathrm{old}}\gets\theta$; $\mathcal J_{\mathcal B}\gets0$
\For{each prompt $q\in\mathcal B$}
  \State Sample $\{y^{(i)}\}_{i=1}^{G}\sim\pi_{\theta_{\mathrm{old}}}(\cdot\mid q)$
  \For{$i=1,\ldots,G$}
    \State Obtain $r_s^{(i)},r_u^{(i)},r_o^{(i)}$ and format gate $F_i$
    \State $\bar r_i\gets(r_s^{(i)}+r_u^{(i)}+r_o^{(i)})/3$
    \If{$q\in\mathcal D_h$}
      \State $r_i\gets F_i r_s^{(i)}$
    \Else
      \State $r_i\gets F_i\min\{r_s^{(i)},\bar r_i\}$
    \EndIf
  \EndFor
  \State $\hat A_i\gets(r_i-\mu_{\mathcal R})/
    (\sigma_{\mathcal R}+\epsilon)$ for $i=1,\ldots,G$
  \State Sort rollouts so $\hat A_{(1)}\geq\cdots\geq\hat A_{(G)}$
  \State Form pairs $\mathcal P_q\gets
    \{(y^{(i)},y^{(G-i+1)})\}_{i=1}^{G/2}$
  \State Retain the $\lceil\kappa G/2\rceil$ pairs with largest
    $\Delta_i=\hat A_{(i)}-\hat A_{(G-i+1)}$
  \State Let $\mathcal T_q$ contain both trajectories in each retained pair
  \State $\mathcal J_q\gets
    \mathcal J_{\mathrm{GRPO}}(\theta;q,\mathcal T_q,\hat A)$
  \If{$q\in\mathcal D_h$}
    \State $\mathcal J_q\gets\mathcal J_q+
      \lambda_{\mathrm{SRS}}\mathcal J_{\mathrm{SRS}}(\theta;q,y_q^*)$
  \Else
    \State $\mathcal K_q\gets D_{\mathrm{KL}}(\pi_\theta(\cdot\mid q)
      \Vert\pi_{\mathrm{base}}(\cdot\mid q))$
    \State $\mathcal J_q\gets\mathcal J_q-\lambda_b\mathcal K_q$
  \EndIf
  \State $\mathcal J_{\mathcal B}\gets\mathcal J_{\mathcal B}+\mathcal J_q$
\EndFor
\State Update $\theta$ by ascending $\mathcal J_{\mathcal B}/|\mathcal B|$
\end{algorithmic}
\end{algorithm}

For each prompt, the policy first generates $G$ rollouts, which receive a
format-gated reward. Harmful prompts are scored by safety, while benign
prompts use the minimum of safety and the mean safety, utility, and honesty
score. Group-normalized advantages then determine the AGTF pairs: the
highest-advantage rollout is paired with the lowest, and so on. Only the
top-$\kappa$ fraction of pairs contributes to the GRPO policy term, with
both members of each selected pair retained. The GRPO term applies the
clipped objective and regularizers defined in the Preliminaries to these
retained trajectories.

The additional objective depends on the prompt category. For a harmful
prompt, SRS maximizes the token-level likelihood of its safe reference,
including the reasoning trace, final answer, and EOS token. It therefore
continues to provide a safety direction when all sampled rewards, and hence
their group-relative advantages, are equal. For a benign prompt, the
benign-only KL term anchors the policy to its pre-alignment behavior to
limit unnecessary refusals. The old sampling policy is refreshed at each
optimization step; $\pi_{\mathrm{base}}$ remains fixed throughout training.

\section{AI Usage Statement}

The initial research idea originated with the authors. During the development
of this work, the authors used AI tools to discuss and refine ideas, improve
the language of the manuscript, and assist with code modifications. The
authors take responsibility for the final manuscript and code.

\end{document}